\def\year{2022}\relax
\documentclass[letterpaper]{article} 
\usepackage{aaai23}  
\usepackage{times}  
\usepackage{helvet}  
\usepackage{courier}  
\usepackage[hyphens]{url}  
\usepackage{graphicx} 
\usepackage{natbib}  
\usepackage{caption} 
\usepackage{algorithm}
\usepackage{algorithmic}
\usepackage{amsmath}
\usepackage{bm}
\usepackage{graphicx}
\usepackage{xspace}
\usepackage{subfigure}
\usepackage{todonotes}
\newcommand{\oomit}[1]{}
\usepackage{newfloat}
\usepackage{listings}
\DeclareCaptionStyle{ruled}{labelfont=normalfont,labelsep=colon,strut=off} 
\floatstyle{ruled}
\newfloat{listing}{tb}{lst}{}
\floatname{listing}{Listing}
\usepackage{amsthm}
\usepackage{amssymb}

\newtheorem{theorem}{Theorem}

\newcommand \nn {\mathcal{N}}

\newcommand \bl {\mathcal{B}}
\newcommand \jac {\mathcal{J}}

\newcommand \ann {ANN\xspace}
\newcommand \bbR {\mathbb{R}}
\newenvironment{remark} {\noindent\textbf{Remark} }{}
\DeclareMathOperator*{\rank}{rank}
\DeclareMathOperator*{\vect}{Vec}
\DeclareMathOperator*{\abs}{abs}

\title{Credit Assignment for Trained Neural Networks Based on \\ Koopman Operator Theory\thanks{This work is supported by the National Natural Science Foundation of China under Grant No.61872371, No.61836005 and No.62032024 and  the CAS Pioneer Hundred Talents Program.}}
\author{
    Zhen Liang$^1$, Changyuan Zhao$^2$, Wanwei Liu$^1$,  Bai Xue$^2$,  Wenjing Yang$^1$ and Zhengbin Pang$^1$
}
\affiliations{
    \textsuperscript{\rm 1}College of Computer Science and Technology, National University of Defense Technology, Changsha, China\\
    Email: \{liangzhen,wwliu,wenjing.yang, zhengbinpang\}@nudt.edu.cn\\
    \textsuperscript{\rm 2}Institute of Software CAS, Beijing, China\\
    Email: \{zhaocy,xuebai\}@ios.ac.cn

}

\usepackage{bibentry}

\begin{document}

\maketitle

\begin{abstract}
Credit assignment problem  of neural networks refers to evaluating the credit of each network component to the final outputs.  For an untrained  neural network, approaches to tackling it have made great contributions to parameter update and model revolution during the training phase. This problem on trained neural networks receives rare attention,  nevertheless, it plays an increasingly important role in neural network patch, specification and verification. Based on Koopman operator theory, this paper presents an alternative perspective of linear dynamics  on dealing with the credit assignment problem for trained neural networks. Regarding a neural network as the composition of sub-dynamics series, we utilize step-delay embedding to capture snapshots of each component, characterizing the established mapping as exactly as possible. To circumvent the dimension-difference problem encountered during the embedding, a composition and decomposition of an auxiliary linear layer, termed minimal linear dimension alignment, is carefully designed with rigorous formal guarantee. Afterwards, each component is approximated by a Koopman operator and we derive the Jacobian matrix and its  corresponding determinant, similar to backward propagation. Then, we can define a metric with algebraic interpretability for the credit assignment of each network component. Moreover, experiments conducted on typical neural networks demonstrate the effectiveness of the proposed method.
\end{abstract}

\section{I Introduction}
\label{sec:intro}
Artificial neural networks (ANN), also known as neural networks (NN) for short,  have recently emerged as leading candidate models for deep learning (DL), popularly used in a variety of areas, such as computer vision \cite{dahnert2021panoptic, esser2021taming, tian2021image}, natural language processing \cite{karch2021grounding, wang2021grad2task, yuan2021bartscore} and so on. Behind the enormous success, {\ann}s are generally with  complicated structures, meaning that there is an intricate data flow through multiple linear or nonlinear components from an input sample to its corresponding output. Since each component works for the information processing and affects the subsequent components, there is a pressing need to evaluate how much a certain component contributes for the final output. This problem is termed Credit Assignment Problem (CAP) \cite{minsky1961steps}, which has being at the heart of the training methods of ANNs since proposed.

For an untrained neural network consisting of a large quantity of parameters,  CAP coincides with the requirement to change parameters at different levels during its training phase.  Backward Propagation algorithm (BP)  \cite{werbos1974beyond} provides a powerful and popular approach to CAP, which computes partial derivatives representing the sensitivity of the neurons in a certain layer  on the final loss and indirectly reflects the credits of different neurons to the network output.  It is widely utilized in the mainstream training methods of {\ann}s, such as  SGD \cite{sgdarticle}, RMSProp \cite{tieleman2012lecture} and Adam \cite{Kingma2015AdamAM}, promoting the prosperity of neural networks and deep learning greatly. 

Whereas, instead of training phases, the problem on trained neural networks, i.e.,  evaluating the credit of each component to the network capability,  receives little focus. As a matter of fact,  the CAP on trained ANNs is of vital significance  in neural network related research. Minimal network patch \cite{kauschke2018batchwise, goldberger2020minimal} requires to find out the minimal modification on network parameters to satisfy certain given properties which do not hold before. Furthermore,  formal specification and verification  on {\ann}s \cite{liu2020verifying, 2019Verification, vengertsev2020recurrent, liu2021algorithms} demands to do specifications or verification on the most sensitive inputs related to robustness, reachability and so on.

Visualization seems to be a candidate solution to CAP on trained {\ann}s and some researchers resort to feature maps to describe the contribution of newly designed neural network components (or, layers) \cite{han2015deep, 2017Visual,tang2018evaluation}. However, it is mainly limited in computer vision domains and lacks of formal analysis and guarantee. Meanwhile,  it is difficult for human to recognize the feature maps as the process goes in networks, let alone matching them with the learning credit, as illustrated in Fig. \ref{fig:lenet}. Consequently, it is instrumental and urgent to deal with the credit assignment problem for trained neural networks formally and rigorously. For simplicity, the following referred CAPs are all on trained neural networks.

In fact, an \ann can be regarded as a special class of nonlinear dynamical system evolved from big data and there exist mature data-driven theories of dynamics analysis, which makes it possible to tackle the credit assignment problem from a dynamics perspective. More appealingly,  Koopman operator theory could accomplish linear approximation of nonlinear systems, which greatly simplifies the analysis on the base of corresponding Koopman operators. Though the Koopman operator is a linear transformation on infinite dimensions, an approximation transformation on  finite dimensions is adopted generally, which can be obtained with the representative data-driven algorithms, such as Dynamic Mode Decomposition (DMD) \cite{kutz2016dynamic}, Extended Dynamic Mode Decomposition (EDMD) \cite{li2017extended}, Kernel Dynamic Mode Decomposition (KDMD) \cite{kawahara2016dynamic} and so on. When a nonlinear dynamical system is associated with a linear transformation, its credit analysis would be tractable with the linear algebra theory then.

In this paper, inspired by the Koopman operator theory and backward propagation process, we propose an alternative approach to resolving the credit assignment problem for trained neural networks from a linear dynamical system perspective. We treat an \ann as a nonlinear system, consisted of a series of sub-dynamics (i.e., components) to be assigned credit. With the aid of Koopman linearization of each subsystem, the established function of an \ann is approximately characterized by the composition of a linear transformation series. Then the CAP of an \ann component is reduced to the contribution of its corresponding Koopman operator to the whole transformation. Similar to the partial derivatives of BP algorithm, we leverage the concepts of Jacobian matrix/determinant to derive and define a metric with algebraic interpretation and further to assign the credit to each network component subsequently.

Main contributions of this paper are listed as follows:

\begin{itemize}
    \item Migrating traditional CAP to trained neural networks and regarding the input/output of an \ann (or, component) as the evolution series of a dynamical system, we utilize step-delay embedding to obtain a more precise representation of the component function and further a more accurate linear approximation with Koopman operator theory.
    
    \item To deal with the dimension-gap between  the input and output layers of a component,  encountered in the step-delay embedding, we present a minimal linear dimension alignment approach via composing and decomposing a linear network layer onto a network component, which provides a novel insight into the dimension difference.
    
    \item Resorting to linear transformation analysis, we derive the sensitivity of a network component to the whole network capability with Jacobian matrices and the corresponding determinants, and define a credit metric with comprehensible algebraic explanation for assigning the learning credits to network components. 
\end{itemize}

The remainder of this paper is organized as follows. We first briefly introduce related preliminaries. Then we give the technique details for solving CAP on trained {\ann}s, including block partition, minimal linear dimension alignment, Koopman approximation and credit assignment.  Subsequently, we apply the presented approach on some typical neural networks to justify its effectiveness with experimental demonstration. We summarize this paper and discuss the possible future work  at the end.

\section{II Preliminaries}
\label{sec:pre}

In this paper, we let $\bbR^{m,n}$ be the space consists of all $m$ by $n$ real matrices, and let $\bbR^{k}$ be the space constituted by real vectors with length $k$. 

Given two matrices $\bm{A}=\left(a_{i,j}\right)_{i\leq m,j\leq n}\in\bbR^{m,n}$ and $\bm{B}\in\bbR^{k,\ell}$, then the \emph{Kronecker product} of $\bm{A}$ and $\bm{B}$,
denoted as $\bm{A}\otimes\bm{B}$, is the matrix $\left(a_{i,j}\cdot \bm{B}\right)_{i\leq m, j\leq n}
\in\bbR^{m\times k, n\times\ell}$. Let $\vect(\bm{A})$ be the vector 
$$(a_{1,1}, \ldots, a_{1,n}, a_{2,1},\ldots,a_{2,n}, \ldots, a_{m,1},\ldots,a_{m,n})^{\rm T}.$$
We in what follows use $\vect^{-1}_{m,n}$ to be its inverse function which restores
a vector (with length $m\times n$) to an $m$ by $n$ matrix, and we sometimes omit the
script $m,n$ if there is no risk of ambiguity. 

Given a mapping $\bm{f}: \bbR^{m,n}\to\bbR^{k,\ell}$, we denote by $\bm{f}^*: \bbR^{mn}\to\bbR^{k\ell}$ its \emph{vectorized function}, defined as
$\bm{f}^* = \vect\circ\bm{f}\circ\vect^{-1}$. In other words, if $\bm{f}(\bm{A})=\bm{B}$
then $\bm{f}^*(\vect(\bm{A}))=\vect(\bm{B})$.

\begin{theorem}
\label{linear}
  Let $\bm{f}(\bm{X})=\bm{AXB}$, where $\bm{A}$ and $\bm{B}$ are two given matrices,
  then $\bm{f}^*(\bm{v})=(\bm{A}\otimes\bm{B}^{\rm T})\bm{v}$.
\end{theorem}

Koopman operator theory works on the linear approximation of nonlinear dynamics \cite{koopman1931hamiltonian, mezic2005spectral, budivsic2012applied}. For a nonlinear dynamical system
$$\bm{x}_{t+1} = \bm{f}(\bm{x}_{t}),$$
Koopman operator theory requires to construct a linear operator $\bm{K}$, so-called \emph{Koopman operator}, satisfying:
$$\bm{K} \bm{\psi}(\bm{x}_{t}) = \bm{\psi} \bm{f}(\bm{x}_{t}), $$
where $\bm{\psi}$ is called the \emph{observation function}, namely,  a set of scalar value functions defined in the state space, utilized to observe the evolution of the system over time. Generally, $\bm{K}$ is an infinite-dimensional  operator.
Practically, numerical techniques on approximating Koopman operator on finite dimensions have been developed,  among which, DMD \cite{kutz2016dynamic} is a widely used algorithm. This method simply depends on collecting evolution series $\bm{x}_i$ of the given dynamical system at time instant $t_i$, where $i\in{1,2,3,\cdots, m}$. DMD works as a regression of evolution series onto locally linear dynamics $\bm{x}_{i+1}=\bm{Kx}_{i}$, where $\bm{K}$ is chosen to minimize $\| \bm{x}_{i+1}-\bm{K}\bm{x}_i\|_2,$ $i\in \{1,2,3,\cdots, m-1\}$. In DMD, the observation function is not explicitly listed, i.e., identity mapping actually works as the  observation function.

\begin{theorem}
Given a matrix $\bm{M}\in \mathbb{R}^{m,n}$, there exists a unique matrix $\bm{X}\in \bbR^{n,m}$ fulfilling the following requirements
\begin{equation}\tag{$\dagger$} \label{PMI}
    \begin{array}{cc}
    \bm{MXM=M}  & \bm{XMX=X} \\
    (\bm{MX})^{\rm T}=\bm{MX} & (\bm{XM})^{\rm T}=\bm{XM}
\end{array}.
\end{equation}
\label{moore}
\end{theorem}
We call such $\bm{X}$ fulfilling \eqref{PMI}
the \emph{Moore-Penrose inverse} of $\bm{M}$ \cite{penrose1955generalized}, denoted as $\bm{M}^{+}$. 
Indeed, when $\bm{M}$ is a non-singular (square) matrix,  $\bm{M}^+$ coincides with $\bm{M}^{-1}$.

Moore-Penrose inverse also features a nice property in solving minimum least square problems.
\begin{theorem}
Suppose that matrix $\bm{M}\in\bbR^{m, n}$, for $\forall \bm{Z}\in\bbR^{n, p}$ and $\forall \bm{N}\in\bbR^{m, p}$,
 \begin{equation*}
     \|\bm{M}\bm{Z}- \bm{N} \|_F \geq \|\bm{M}\bm{M}^{+}-\bm{N}\|_F.
 \end{equation*}
i.e., 
$\bm{M}^{+}$ is an optimal solution of the optimization problem $\|\bm{M}\bm{Z}- \bm{N} \|_F$, where $\|\bullet\|_F$ denotes the Frobenius norm of a matrix \cite{james1978generalised, planitz19793}.
 \label{minimal}
\end{theorem}

Moore-Penrose inverse of $\bm{M}$ can be computed as follows: Suppose that the singular value decomposition (SVD) of $\bm{M}$ is $\bm{U} \left(\begin{matrix} \bm{\Sigma} & \bm{0} \\ \bm{0} & \bm{0} \end{matrix}\right) \bm{V}^{T}$, where $\bm{U}\in\bbR^{m,m}$, 
$\bm{V}\in\bbR^{n,n}$ are unitary matrices, and
$\bm{\Sigma} ={\rm{diag}}(\sigma_1,\ldots,\sigma_r)$ with each $\sigma_i>0$,
then we have $\bm{M}^{+} = \bm{V} \left(\begin{matrix} \bm{\Sigma}^{-1} & \bm{0} \\ \bm{0} & \bm{0} \end{matrix}\right) \bm{U}^{T}$.

For a  multivariate vector-valued function $\bm{f}: \mathbb{R}^{n}\rightarrow \mathbb{R}^{m}$, its 
\emph{Jacobian matrix}  
is defined as  \cite{jacobi1841determinantibus}:
\begin{equation*}
    \jac_{\bm{f}}= \left[\frac{\partial{\bm{f}}}{\partial{x_1}}, \frac{\partial{\bm{f}}}{\partial{x_2}},  \cdots, \frac{\partial{\bm{f}}}{\partial{x_n}}\right]  = \begin{bmatrix}
    \frac{\partial{f_1}}{\partial{x_1}} & \frac{\partial{f_1}}{\partial{x_2}} & \cdots & \frac{\partial{f_1}}{\partial{x_n}} \\
    \frac{\partial{f_2}}{\partial{x_1}} & \frac{\partial{f_2}}{\partial{x_2}} & \cdots & \frac{\partial{f_2}}{\partial{x_n}} \\
    \vdots  & \vdots & \ddots & \vdots\\
    \frac{\partial{f_m}}{\partial{x_1}} & \frac{\partial{f_m}}{\partial{x_2}} & \cdots & \frac{\partial{f_m}}{\partial{x_n}}
        \end{bmatrix}.
\end{equation*}


Suppose $\bm{f}:\bbR^{m,n}\to\bbR^{k,\ell}$ and $\bm{Y}=\bm{f}(\bm{X})$, then we denote
$\frac{\partial\bm{Y}}{\partial\bm{X}}$ the Jacobian matrix $\jac_{\bm{f}^*}$, recall that
$\bm{f}^*$ is the vectorized mapping of $\bm{f}$.

\section{III Methodology}
\label{sec:method}
In this section, we illustrate the approach to tackling the credit assignment problem of trained {\ann}s from a linear dynamics perspective. We begin with block partition to divide an \ann into some smaller blocks for scalability. Then before utilizing the Koopman operator theory to approximate blocks with linear transformation, a minimal linear dimension alignment is proposed to be compatible with step-delay embedding. Finally, we assign the credit of each block with Jacobian matrices of the Koopman operators. 

For presentation brevity, a general neural network $\nn$ with $l+1$ layers is  used throughout this section. Thus, the function $f$ of neural network $\mathcal{N}$ can be represented by the composition of the transformation of each layer, i.e.,
\begin{equation*}
\begin{aligned}
    f(\bm{x}) &= f_l\circ f_{f-1} \circ \cdots \circ f_{2} \circ f_{1}(\bm{x}) \\
    &=\sigma_{l}(\bm{W}_{l}\cdots \sigma_{2}(\bm{W}_2\sigma_1(\bm{W}_1 \bm{x} + \bm{b}_1)+\bm{b}_2)\cdots +\bm{b}_{l})
\end{aligned}
     \label{nn}
\end{equation*}
where $\bm{W}_i$ and $\bm{b}_i$, $1 \leq i \leq l$, respectively stand for the weight matrix and bias  in the adjacent $i$-$th$ and $i+1$-$th$ layer, completing the so-called \emph{affine transformation}. $\sigma_i$ represents the corresponding \emph{activation function} following the affine transformation, such as ReLU, Sigmoid, Tanh and so on, which results in the non-linearity of neural networks. For a given input vector $\bm{x}$, there would generate a vector flow $\bm{x}_0($i.e.,$ \bm{x}), \bm{x}_1, \cdots, \bm{x}_l$,  meanwhile, for $i$-$th$ network layer, it establishes a mapping defined as $\bm{x}_{i} = \sigma_{i-1}(\bm{W}_{i-1} \bm{x}_{i-1}+\bm{b}_{i-1})$, from its input $\bm{x}_{i-1}$ to its output $\bm{x}_{i}$.

\subsection{A. Block Partition.}
Block partition refers to dividing a neural network $\nn$ into some blocks.  In essence,  block partition is put forward out of benign scalability and it is a divide-and-conquer technique. There are two main  concerns for block division: on the one hand, block partition allows us to tackle neural networks with large scale, because we only concentrate on the inputs and outputs of blocks, whereas the inside of the blocks is treated as a black box;  on the other hand,  during this process, we can highlight the layers or components interested in with finer partition  and  aggregate others in a coarser style.

Moreover, blocks match well with the basic structure of neural networks. Generally, a block contains some layers of network $\nn$, such as in feed-forward neural networks (FFNN)  and convolution neural networks (CNN) \cite{bishop2006pattern}. In recurrent neural networks (RNN) \cite{gers2000learning} or residual networks (ResNet) \cite{he2016deep}, better still, layers are divided into some network modules, which have been considered during the design process, providing great convenience and guidance for block partition.

Herein, we partition the neural network $\nn$ into $m$ blocks $\bl_1, \ \bl_2, \ \cdots, \ \bl_m$, where $ 1 \leq m \leq l$ and each block contains the mapping established by a network layer or the composition of established mappings behind some successive layers. More concretely, suppose the block $\bl_i$ contains the $j$-$th$ to $k$-$th$ layer, then the function corresponding to $\bl_i$ is defined as: $f_i=\sigma_{k}(\bm{W}_{k}\cdots\sigma_{j}(\bm{W}_j\sigma_{j-1}(\bm{W}_{j-1} \bm{x}_{j-1}+\bm{b}_{j-1})+\bm{b}_j)\cdots+\bm{b}_{k})$. The number of the network layers contained in block $\bl_i$ is specified with ${\rm{len}}(\bl_i)$ and a certain layer in $\bl_i$ is denoted by $\bl_{ij}$, $j\in \{1,2, \cdots, {\rm{len}}(\bl_i)\}$. Besides, ${\rm{dim}}(\bl_{ij})$ describes the dimension of the layer $\bl_{ij}$, namely, the number of neurons in it.

\paragraph{Remark.} Block partition makes it possible to solve the large scale neural networks, with partitioning larger blocks into smaller size ones for the next processing iteration.

\subsection{B. Minimal Linear Dimension Alignment.}
When block partition finishes, one can do  linearization upon each subsystem $\bl_i$, $i\in\{1,2,\cdots, m\}$ with Koopman operator theory. Nevertheless, before that, we should beware that such blocks can be categorized into variant types according to their input/output dimensions. For the sake of step-delay embedding discussed in the next subsection, we need do some slight `surgeries'
beforehand, explained below.

\begin{figure}[htbp]
\centering
\subfigure[S2L.]{
\begin{minipage}[t]{0.32\linewidth}
\centering
\includegraphics[width=0.75in]{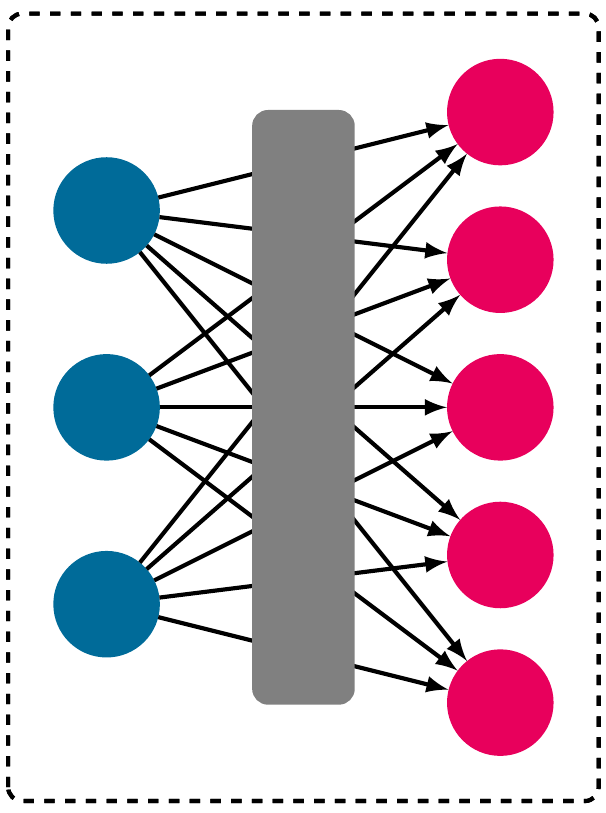}
\label{s2l}
\end{minipage}%
}%
\subfigure[L2S.]{
\begin{minipage}[t]{0.32\linewidth}
\centering
\includegraphics[width=0.77in]{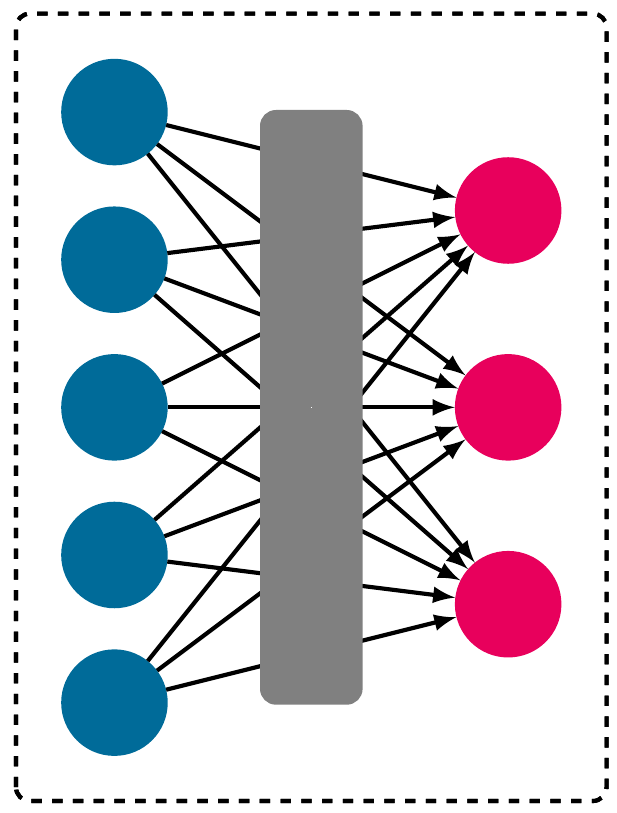}
\label{l2s}
\end{minipage}%
}%
\subfigure[E2E.]{
\begin{minipage}[t]{0.32\linewidth}
\centering
\includegraphics[width=0.77in]{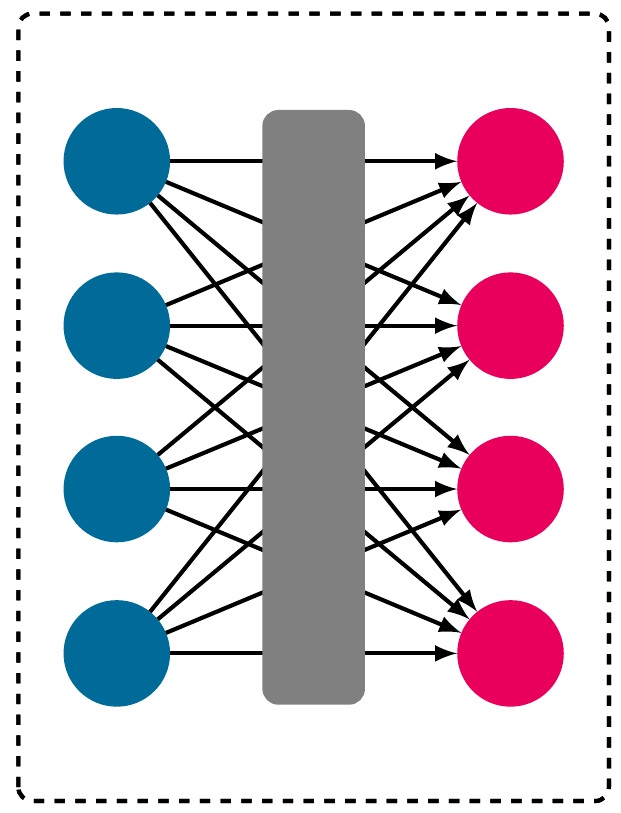}
\label{e2e}
\end{minipage}
}
\centering
\caption{Demonstrations of typical block types.}
\label{block type}
\end{figure}

For a block $\bl_i$,  it ought to establish a mapping $f_i$, yet we here treat it as a black box, and are more concerned about its input-output dimensions. According to them, the blocks can be categorized into three types: as shown in Fig.  \ref{block type}, from Fig. \ref{s2l} to \ref{e2e}, are respectively called 
\begin{description}
\centering
\item[S2L]: when $ {\rm{dim}}(\bl_{i0}) < {\rm{dim}}(\bl_{i,{\rm{len}}({\bl}_i)})$;
\item[L2S]: when ${\rm{dim}}(\bl_{i0}) > {\rm{dim}}(\bl_{i,{\rm{len}}(\bl_i)})$;
\item[E2E]: when ${\rm{dim}}(\bl_{i0}) = {\rm{dim}}(\bl_{i,{\rm{len}}(\bl_i)})$.
\end{description}
 In Fig.\ref{block type}, only the input and output layers of the blocks are explicitly exhibited and the hidden layers are abstracted by gray bars.

When a batch of input samples are fed into $\nn$, we obtain the input and output of each block, called a \emph{snapshot} at one step. To characterize the mapping $f_i$ behind in $\bl_i$ more detailedly, we utilized the step-delay embedding technique to get more snapshots, which is illustrated in Subsection II.C and requires to feed the output into blocks iteratively. For L2S and/or S2L blocks, the gap between input and output dimensions bring about difficulties in step-delay embedding. To circumvent it, we present the so-called  \emph{minimal linear dimension alignment method} to replace a L2S/S2L block with some E2E block which is most `approximated' to the original block in some sense. Fig. \ref{dim align 1}  and \ref{dim align 2} schematically demonstrate the idea for dealing with S2L blocks and L2S blocks respectively. Here, we have explicitly separated the affine transformation and the activation within the block, and to emphasize the goal of dimension alignment is to compose it with some linear transformation layer, denoted by matrix $\bm{A}$, making the input/output dimensions coincide (as shown in Fig. \ref{dim align}). Notably, we also need to decompose the added linear transformation layer at last, with another linear transformation, namely, matrix $\bm{B}$. Meanwhile, we require that the composition and the decomposition of  added component affects the original blocks as `minimal' as possible, and it should be as simple as possible to construct in practice, which is formulated as follows.

\begin{figure}[htbp]
\centering
\subfigure[Alignment of S2L case.]{
\begin{minipage}[t]{0.5\linewidth}
\centering
\includegraphics[width=1.4in]{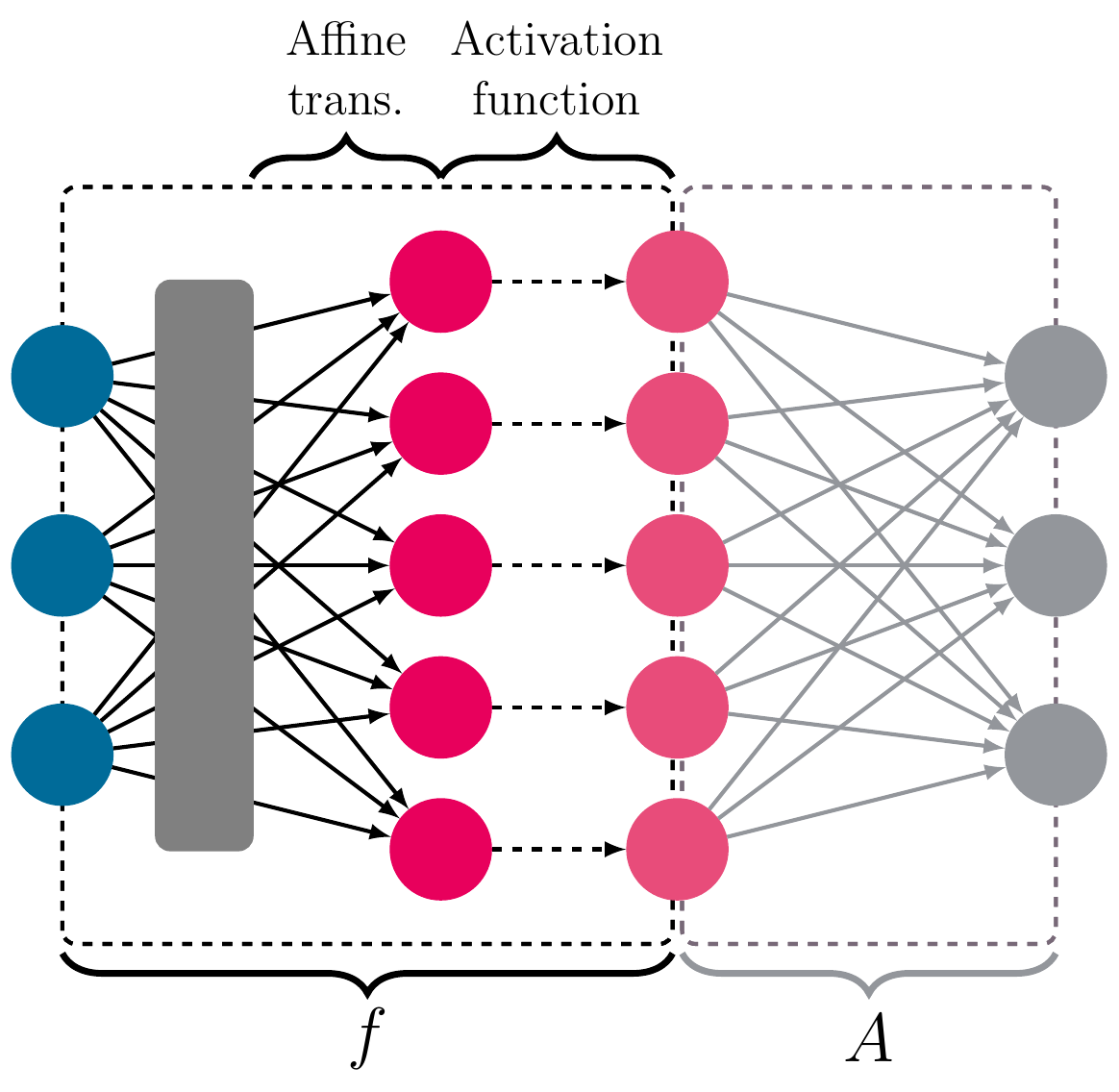}
\label{dim align 1}
\end{minipage}%
}%
\subfigure[Alignment of L2S case.]{
\begin{minipage}[t]{0.5\linewidth}
\centering
\includegraphics[width=1.4in]{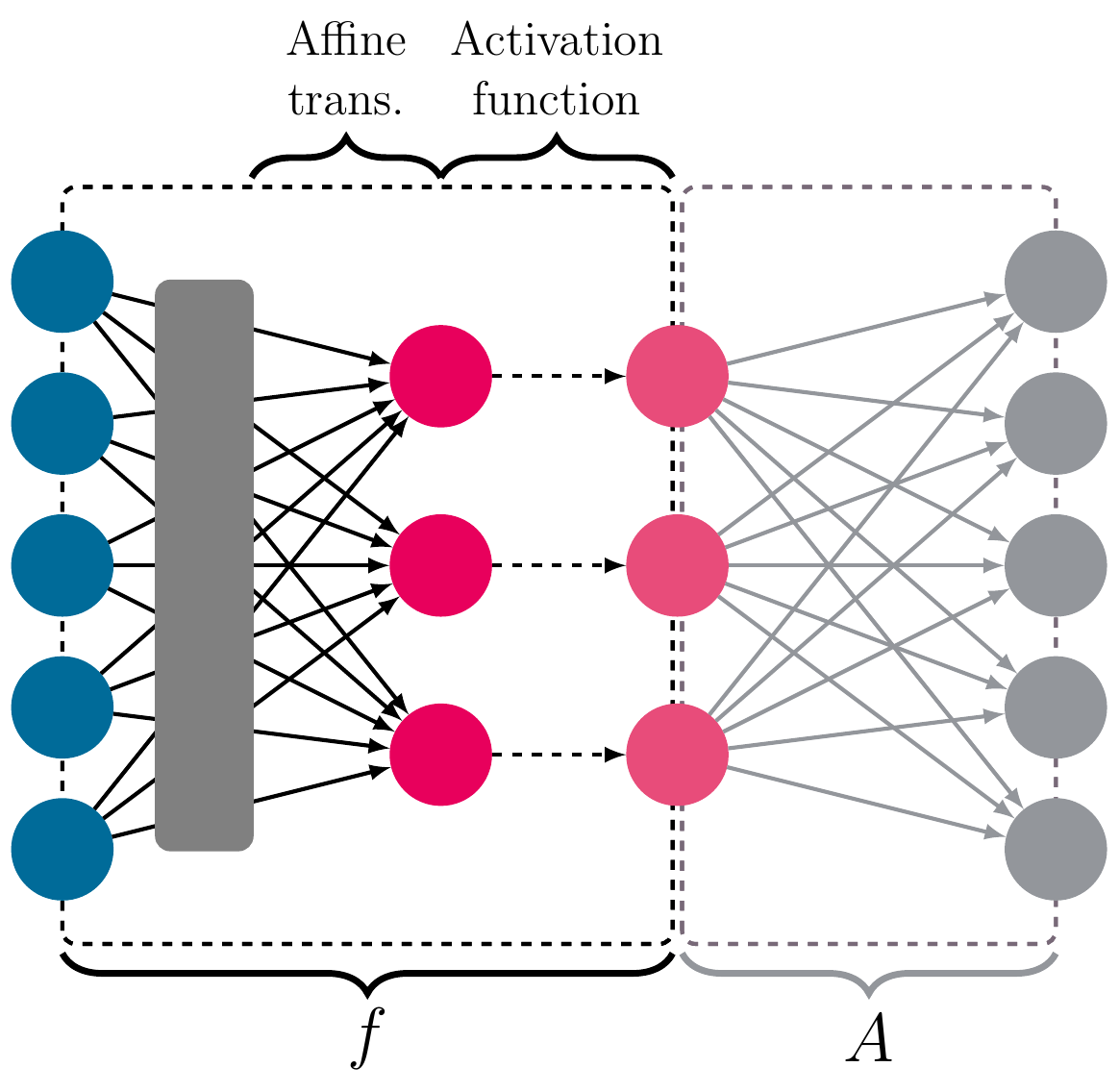}
\label{dim align 2}
\end{minipage}%
}%
\centering
\caption{ Minimal linear dimension alignment.}
\label{dim align}
\end{figure}

\paragraph{Minimal linear composition and decomposition.} Suppose that the matrix representing the alignment layer is $\bm{A}\in\bbR^{m,n}$ and its decomposition is denoted by matrix $\bm{B}\in\bbR^{n, m}$, `minimal composition and decomposition' requires that the composition transformation $\bm{BA}$ approximates identify mapping $\bm{I}$ as exactly as possible, i.e., making the objective function $\|\bm{BA}-\bm{I}_n\|_F$ minimal, recall that  $\|\bullet\|_F$ denotes the Frobenius norm of the matrix.
For this, we have the following theorem to reveal the constrcution of matrix $\bm{A}$ and $\bm{B}$.

\begin{theorem}
    If $\bm{B}\in\bbR^{n, m}$ is a matrix with ${\rm{rank}}(\bm{B})=\min{(m,n)}$ and  $\bm{A} = \bm{B}^{+}$, then $\|\bm{BA}-\bm{I}_n\|_F$ obtains the minimal value $\sqrt {\max ((n-m), 0)}$.
    
    \label{the align}
\end{theorem}


\begin{proof}
 W.l.o.g., suppose that $\bm{B}\in\bbR^{m, n}$ and $\forall \bm{A}\in\bbR^{n, m}$, from Theorem \ref{minimal} we have
\begin{equation*}
     \|\bm{B}\bm{A}- \bm{I}_n \|_F \geq \|\bm{B}\bm{B}^{+}-\bm{I}_n \|_F.
     \label{fromBasic}
 \end{equation*}
Further, let $\rank(\bm{A})=r$, then:
 \begin{equation*}
      \begin{array}{rl}
     \bm{B}\bm{B}^{+}- \bm{I}_n    & =  \bm{V} \left(\begin{matrix} \bm{\Sigma} & \bm{0} \\ \bm{0} & \bm{0} \end{matrix}\right)  \left(\begin{matrix} \bm{\Sigma} & \bm{0} \\ \bm{0} & \bm{0} \end{matrix}\right)^{-1} \bm{V}^{T}  -\bm{I}_n \\
            & =\bm{V} \left(\begin{matrix} \bm{I}_r & \bm{0} \\ \bm{0} & \bm{0} \end{matrix}\right) \bm{V}^{T} - \bm{V} \bm{I}_n \bm{V}^{T} \\
           & = \bm{V} \left(\begin{matrix} \bm{0} & \bm{0} \\ \bm{0} & -\bm{I}_{n-r} \end{matrix}\right)   \bm{V}^{T},
      \end{array}
  \end{equation*}
and we subsequently have
\begin{equation*}
      \begin{array}{ll}
        ||\bm{B}\bm{B}^{+}- \bm{I}_n||_{F}  & = \sqrt{ {\rm{tr}}((\bm{B}\bm{B}^{+}- \bm{I})^{T}(\bm{B}\bm{B}^{+}- \bm{I}))}\\
        & = \sqrt{\rm{tr}\left(\begin{matrix} \bm{0} & \bm{0} \\ \bm{0} & \bm{I}_{n-r} \end{matrix} \right)}\\
        & = \sqrt{n-r}.
      \end{array}
      \label{eq f}
  \end{equation*}
Hence, when taking $r=\min(m,n)$,
\begin{equation*}
    \|\bm{B}\bm{B}^{+}-\bm{I}_n\|_F = \sqrt{\max((n-m),0)},
\end{equation*}
and we have
\begin{equation*}
\begin{array}{ll}
\|\bm{B}\bm{A}- \bm{I}_n \|_F &\geq \|\bm{B}\bm{B}^{+}-\bm{I}_n \|_F \\ & = \sqrt{\max((n-m),0)}.   
\end{array}
\end{equation*}
Thus, the proof is completed.

\end{proof}

Theorem \ref{the align} reveals the construction requirements of  matrices $\bm{A}$ and $\bm{B}$, that is to say, $\bm{B}$ should be  column or row full rank, and  $\bm{A}$ is the Moore-Penrose inverse of matrix $\bm{B}$.  With Theorem \ref{moore}, $\bm{A}$ uniquely exists.

\begin{remark}
By the way, from the conclusion of Theorem \ref{the align}, we are more preferable to the L2S and E2E blocks. In the case that a S2L block is inevitable, try to find a minimal dimension difference when doing the block partition phase. 
\end{remark}

The blocks after completing the minimal linear dimension alignment are denoted by $\tilde{\bl_i}$, $i\in\{1,2,\cdots, m\}$.

\subsection{C. Koopman Approximation with Step-delay Embedding.}
Now, regarding each block $\tilde{\bl_i}$ as a dynamical subsystem, Koopman operator yields a strong theoretical support for linearization of the hidden function $\tilde{f_i}$ corresponding to $\tilde{\bl_i}$. The approximation procedure is carried out in three phases, namely, \emph{step-delay embedding}, \emph{Koopman approximation} and \emph{linear decomposition}.

\paragraph{Step-delay Embedding.}
Notably, once the block $\tilde{\bl_i}$ is obtained from the network $\nn$ (possibly with dimension alignment), the corresponding mapping $\tilde{f_i}$ is determined at the same time, having nothing to do with concrete inputs but only the input-output pairs. To characterize the mapping $\tilde{f_i}$ more precisely, we need to do enough iterations to get snapshots of it, which requires to obtain the evolution of successive steps of block $\tilde{\bl_i}$. Doing series evolution is to feed the block with the current output to get the next output, implying the same dimension of input and output, and hence dimension alignment is required. According to Whitney theorem (or, Takens theorem) \cite{whitney1936differentiable, takens1981detecting, sauer1991embedology} and its application in delay embedding, such as delay-coordinate embedding  \cite{ takeishi2017learning} and time-delay embedding \cite{brunton2017chaos}, it yields a theoretical guarantee for the number of required iterations that preserves the structure of the state space corresponding to network block $\bl_i$, i.e., $2d+1$.  This procedure is shown in Fig. \ref{fig:delay}, termed step-delay embedding herein. As for the value of hyper-parameter $d$, we choose the minimal dimension of the network layer in $\tilde{\bl_i}$,
\begin{equation*}
    d := \min \{ {\rm{dim}}(\tilde{\bl}_{ij}) \ |\ j \in \{1,2,\cdots, {\rm{len}}(\tilde{\bl_i})\} \}
\end{equation*}
as the candidate value for the reason that these dimensions are enough to characterize the features for a trained and well-performed network component, according to manifold learning \cite{tenenbaum2000global}. A faithful series revolution representation  is achieved via the step-delay embedding then, with the following form:
\begin{equation*}
\bm{y}(0),\ \bm{y}(1),\ \cdots,\ \bm{y}(2d), 
\end{equation*}
and
\begin{equation*}
\bm{y}(k+1) = \tilde{f_i}(\bm{y}(k)),\ k\in \{0,1,2,\cdots, 2d-1\}.    
\end{equation*}

\begin{figure}[htbp]
    \centering
    \includegraphics[scale=0.4]{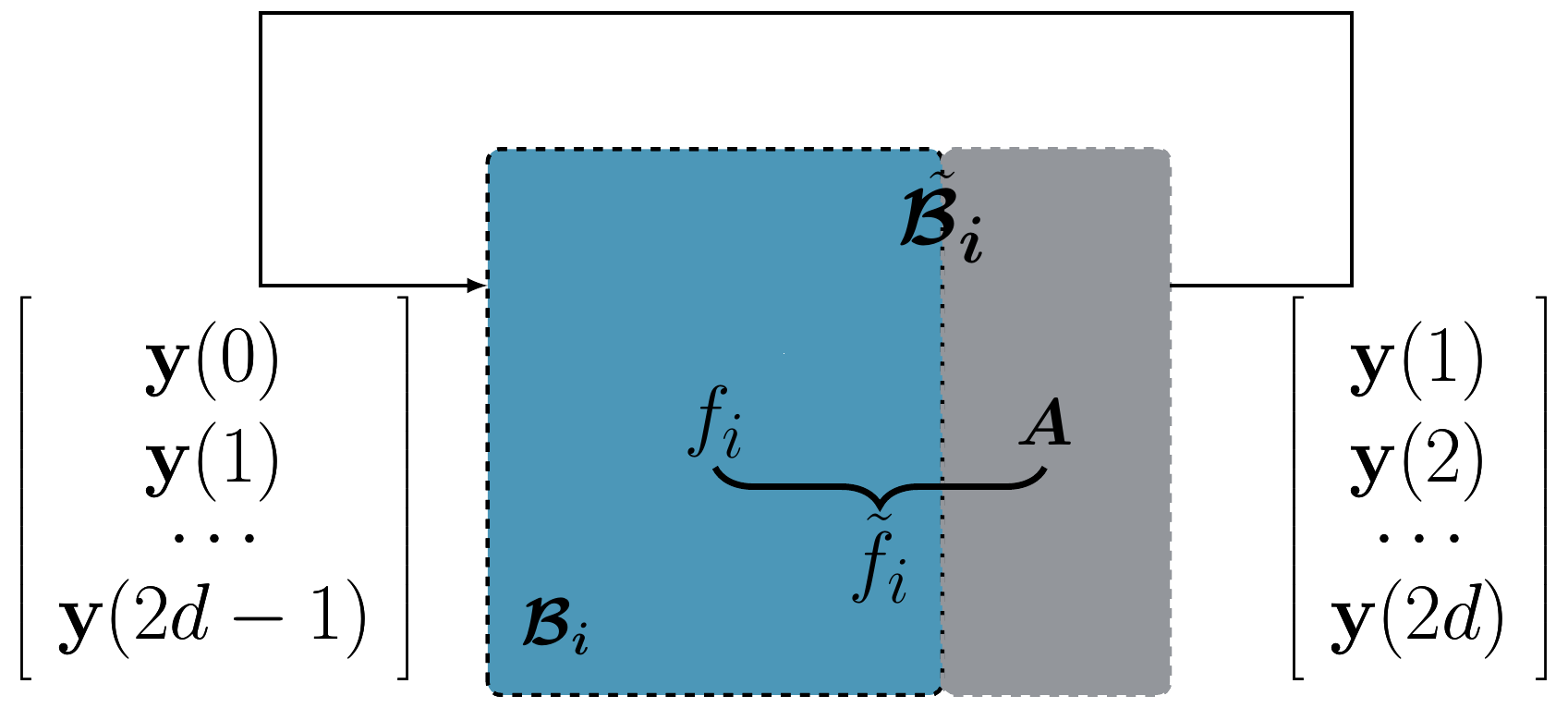}
    \caption{Step-delay embedding.}
    \label{fig:delay}
\end{figure}

\paragraph{Koopman Approximation.}
The goal of Koopman approximation is to construct a linear operator which could be viewed as an approximation of a given nonlinear mapping. 
Herein, as aforementioned,  we take the identity function as observation function on the dynamics variables, i.e., inputs and outputs of network blocks, and utilize  Dynamic Mode Decomposition (DMD) algorithm to learn the linear operator from snapshots. 
Suppose that $\tilde{\bm{K}_i}$ is the Koopman operator obtained from DMD, then we have
\begin{equation*}
    \bm{y}(k+1) \approx {\tilde{\bm{K}}_i}\bm{y}(k), \  k\in \{0,1,2,\cdots, 2d-1\}. 
\end{equation*}

\paragraph{Linear Decomposition.}
Since block $\tilde{\bl_i}$ is consisted with the block $\bl_i$ and an auxiliary linear layer denoted by $\bm{A}$, to recover the hidden function of $\bl_i$, we need to eliminate the effect of the linear layer from $\tilde{\bm{K}_i}$. 
According to Theorem \ref{the align}, we have
\begin{equation}
    \bm{K}_i  \approx {\bm{B}} {\bm{A}} \bm{K}_i =  {\bm{B}} \tilde{\bm{K}_i}.
    \label{eq de}
\end{equation} 
Then the alignment layer is decoupled from $\tilde{\bl_i}$ with Eqn. \eqref{eq de} and the final Koopman operator associated with block $\bl_i$ is obtained.

\subsection{D. Learning Credit Assignment.}
For a network $\nn$, each block  can be featured by a Koopman operator and then the network can be represented by the composition of these operators.  Consequently, the credit assignment of each block is reduced into the contribution of the corresponding matrix to the whole transformation. Suppose that $\bm{K}=\bm{K}_m\bm{K}_{m-1} \cdots \bm{K}_2 \bm{K}_1$, what we need to do is to evaluate the contribution of each $\bm{K}_i$ to $\bm{K}$.

By the aid of backward propagation and Jacobian matrix,  we can obtain the  `partial derivative of $\bm{K}$ w.r.t. $\bm{K}_i$' and we may take the absolute value of its determinant as the block sensitivity of  $\bm{K}_i$, which describes the effect of the change of $\bm{K}_i$ on $\bm{K}$ and is  declared as follows.

\paragraph{Block Sensitivity (BS)} The \emph{block sensitivity} $\bl_i$ is defined as the absolute value of the determinant corresponding to partial derivative of $\bm{K}$ w.r.t. $\bm{K}_i$, i.e,  BS$_i:=\abs\left( \left | \frac{\partial  \bm{K}}{\partial \bm{K_i}}\right |\right)$. 

The determinant of the corresponding Jacobian matrix is derived by the following theorem.

\begin{theorem}
  Let  $\bm{K}=\bm{K}_m\bm{K}_{m-1} \cdots \bm{K}_2 \bm{K}_1$ and suppose that all the matrices here are  all  square matrices with order $n$, then $\left | \frac{ \partial  \bm{K}}{\partial \bm{K}_i }\right | = \prod \limits_{s\neq i} \left | \bm{K}_s \right |^{n}$,  $(1 \leq i \leq m)$.
  \label{the bs}
\end{theorem}

\begin{proof}
According to the Theorem \ref{linear}, we have

\begin{equation}
\left |\frac{\partial  \bm{K}}{\partial \bm{K}_i}\right | =
\begin{cases}
\left | \bm{I}_n \otimes (\prod \limits_{t=2}^{m} \bm{K}_t)^\mathrm{T} \right | &,  \text{$i=1$} \\
\left | \prod \limits_{s=1}^{i-1} \bm{K}_s \otimes (\prod \limits_{t=i+1}^{m} \bm{K}_t)^\mathrm{T} \right |&, \text{$2 \leq i \leq m-1$}\\
\left | \prod \limits_{s=1}^{m-1} \bm{K}_s \otimes \bm{I}_n \right |  &,  \text{$i=m$} 
\end{cases}.
\label{Jacobian}
\end{equation}
Further, for the determinant of Kronecker product can be decomposed as follow,
\begin{equation} 
\lvert \bm{M}\otimes \bm{N} \rvert = \lvert\bm{M}\rvert ^p \times \lvert \bm{N} \rvert^q
\label{kdet}
\end{equation}

where $p$ and $q$ are the orders of the corresponding matrices (both are square matrices). Then combine Eqn.\eqref{Jacobian} and  Eqn.\eqref{kdet} together, we can compute $\left | \frac{\partial \bm{K}}{\partial \bm{K_i}} \right | $ as following

\begin{equation*}
\left | \frac{ \partial  \bm{K}}{\partial \bm{K}_i }\right | =
\begin{cases}
  \left | (\prod \limits_{t=2}^{m} \bm{K}_t)^\mathrm{T} \right |^{n}   &,  \text{$i=1$} \\
\left | \prod \limits_{s=1}^{i-1} \bm{K}_s \right| ^{n} \times  \left| (\prod \limits_{t=i+1}^{m} \bm{K}_t)^\mathrm{T} \right |^{n}  &, \text{$2 \leq i \leq m-1$}\\
\left | \prod \limits_{s=1}^{m-1} \bm{K}_s\right | ^{n} &,  \text{$i=l$} 
\end{cases},
\label{KJacobian}
\end{equation*}
and we thus have
\begin{equation}
\left | \frac{ \partial  \bm{K}}{\partial \bm{K}_i }\right| = \prod \limits_{s\neq i} \left | \bm{K}_s \right |^{n} ,  \quad \text{$1 \leq i \leq m$}
\label{FJacobian}
\end{equation}
\end{proof}

\begin{theorem}
When the Koopman matrix of network block determines, the specific position of network block does not affect the block sensitivity.
\end{theorem}

\begin{proof}
From Eqn.\eqref{FJacobian}, it obviously holds.
\end{proof}

The larger the sensitivity of a block, the more it is affected by other blocks. Further, the sensitivity can be utilized to represent the credit, as in BP process, they are inversely proportional. According to Eqn.\eqref{FJacobian}, we have $\frac{{\rm{BS}}_{i}}{{\rm{BS}}_{j}} \propto {\rm{abs}} \left(\frac{|\bm{K}_{j}|}{|\bm{K}_{i}|}\right) $, and the ratio of their learning credits is proportional to $\abs \left(\frac{|\bm{K}_{i}|}{|\bm{K}_{j}|}\right) $. Therefore, the learning credit of block $\bl_i$ is proportional to ${\rm{abs}}\left(|\bm{K}_i|\right)$. Therefore, the credit of a network block can be defined as follows.

\paragraph{Block Credit (BC)} The \emph{block credit}  ${\rm{BC}}_i$ of block  $\bl_i$ can be defined by the absolute value of the determinant  of corresponding Koopman operator, namely, ${\rm BC}_i \propto{{\rm{abs}}\left(|\bm{K}_i|\right)}$.

\paragraph{Dealing with non-square matrices.} Now, we generalize square matrices in Theorem \ref{the bs} to non-square ones. Recall that the absolute value of the Jacobian determinant presents the transformation scale ratio of an unit volume before and after transformation, and the determinant of non-square matrices should remain the same meaning. For a non-square matrix $\bm{M}\in\bbR^{m, n}$, from standard theory of linear algebra, we have
 \begin{equation*}
\abs(\left | \bm{M} \right |) =
\begin{cases}
   \sqrt{\left | \bm{M}^{\rm T}\bm{M} \right |},   &  \text{$m>n$} \\
    \sqrt{\left | \bm{M}\bm{M}^{\rm T} \right |},   &  \text{$m<n$}
\end{cases}.
\end{equation*}
It satisfies the geometrical requirement with detailed proof in \cite{gover2010determinants}.

\paragraph{Remark.}In fact, it also enlightens us a quick algorithm, we may carry out the PCA reduction (principal components analysis) method \cite{wold1987principal}  to extract the most important transformation features with the same dimension to characterize the credit, avoiding the dimension alignment and non-square cases. Notably, in some cases, the determinant may be 0. According to the analysis in preliminaries, this indicates that information loss occurs in some dimensions. We use the product of non-zero eigenvalues to describe the volume ratio before and after transformation, retaining the same algebraic explanation.

\paragraph{Another view on $\bm{K}$.} From a global perspective, we take the whole transformation between the input $\bm{x}$ and output $\bm{y}$ of network $\nn$, i.e, $\bm{K}$,  into consideration. Starting with 
$$    \bm{y = K x}, $$
w.l.o.g., assume that the shapes of $\bm{x} \in \bbR^{n}$, $\bm{y}\in \bbR^{m}$ and $\bm{K}\in \bbR^{m,n}$, then
$$ \bm{y}^i = \bm{K}^i \bm{x} =  \sum\nolimits_{j=1}^{n} \bm{K}^{ij} \bm{x}^j $$
sheds light on feature importance of input vector $\bm{x}$. For input $\bm{x}$, the importance (i.e., credit) of feature $\bm{x}^j$ to $\bm{y}^i$ is the combinational weight $\bm{K}^{ij}$. What's more, the above properties can also be found for each $\bm{K}_i$, which further reveals the credit of certain neurons.

\begin{remark}
Backward propagation seems to be a feasible approach to the CAP for trained neural networks.  Nevertheless, there are some reasons from two main aspects that result in its inapplicability. On one hand, the centre  of BP process is the chain rule, that is to say, the sensitivity is only related to the part between the current layer and the output layer, and it is irrelevant to other components. On the other hand, the derivatives of BP describe the sensitivity of some hyper-parameter to the final output loss, not the local block capability to the whole network function. Therefore, it is inapplicable on this problem and we do not do comparison with the BP algorithm in what follows.
\end{remark}

\begin{figure*}[htbp]
	\centering
	\subfigure[Credit assignment of layers and convolution kernels in network Lenet-5. \quad ]{
		\begin{minipage}[b]{0.6\linewidth}
			\includegraphics[width=4in]{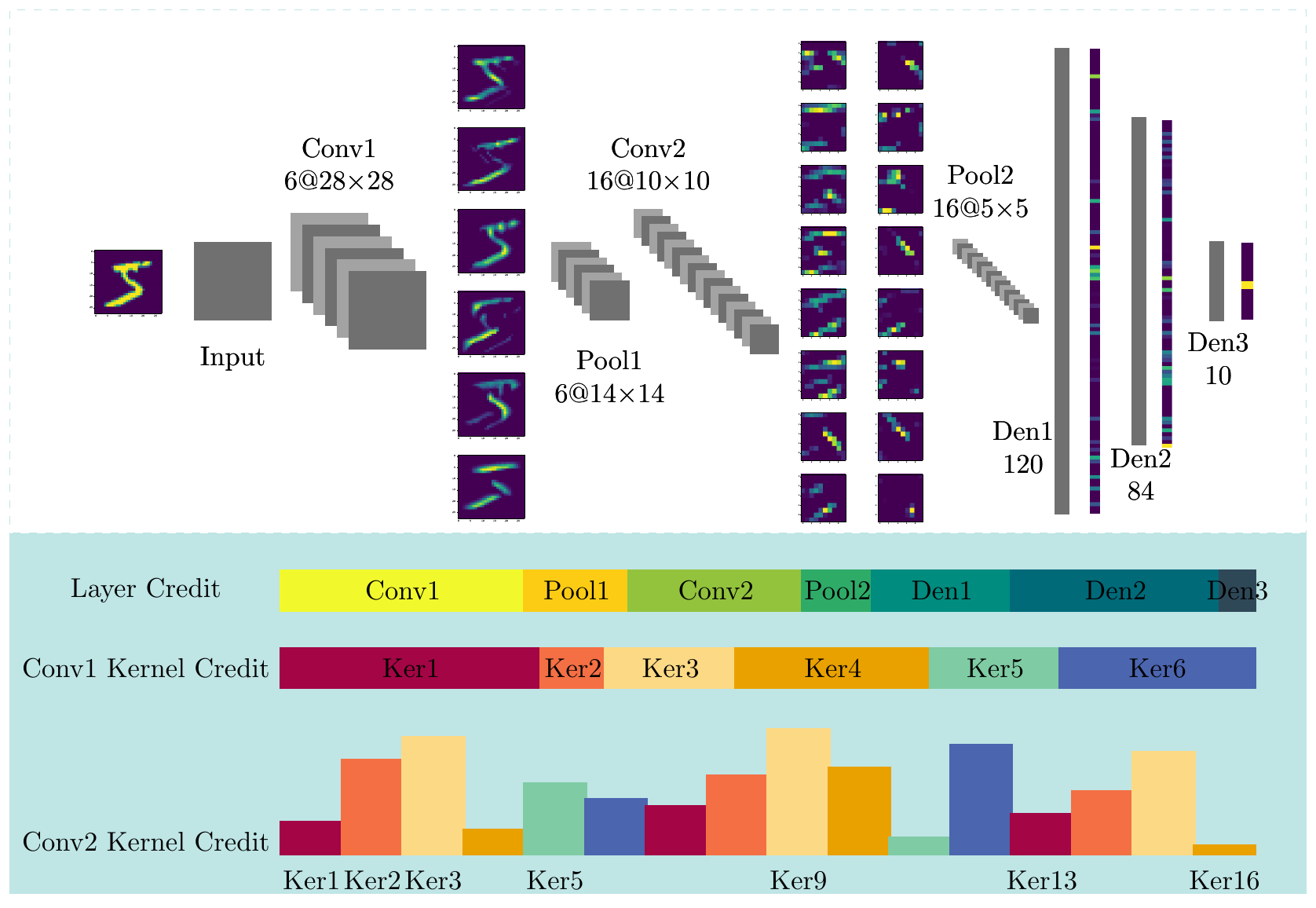}
			\label{fig:lenet}
		\end{minipage}
	} 
	\subfigure[Heat maps of feature weight.]{
		\begin{minipage}[b]{0.36\linewidth}
			\includegraphics[width=1.2in]{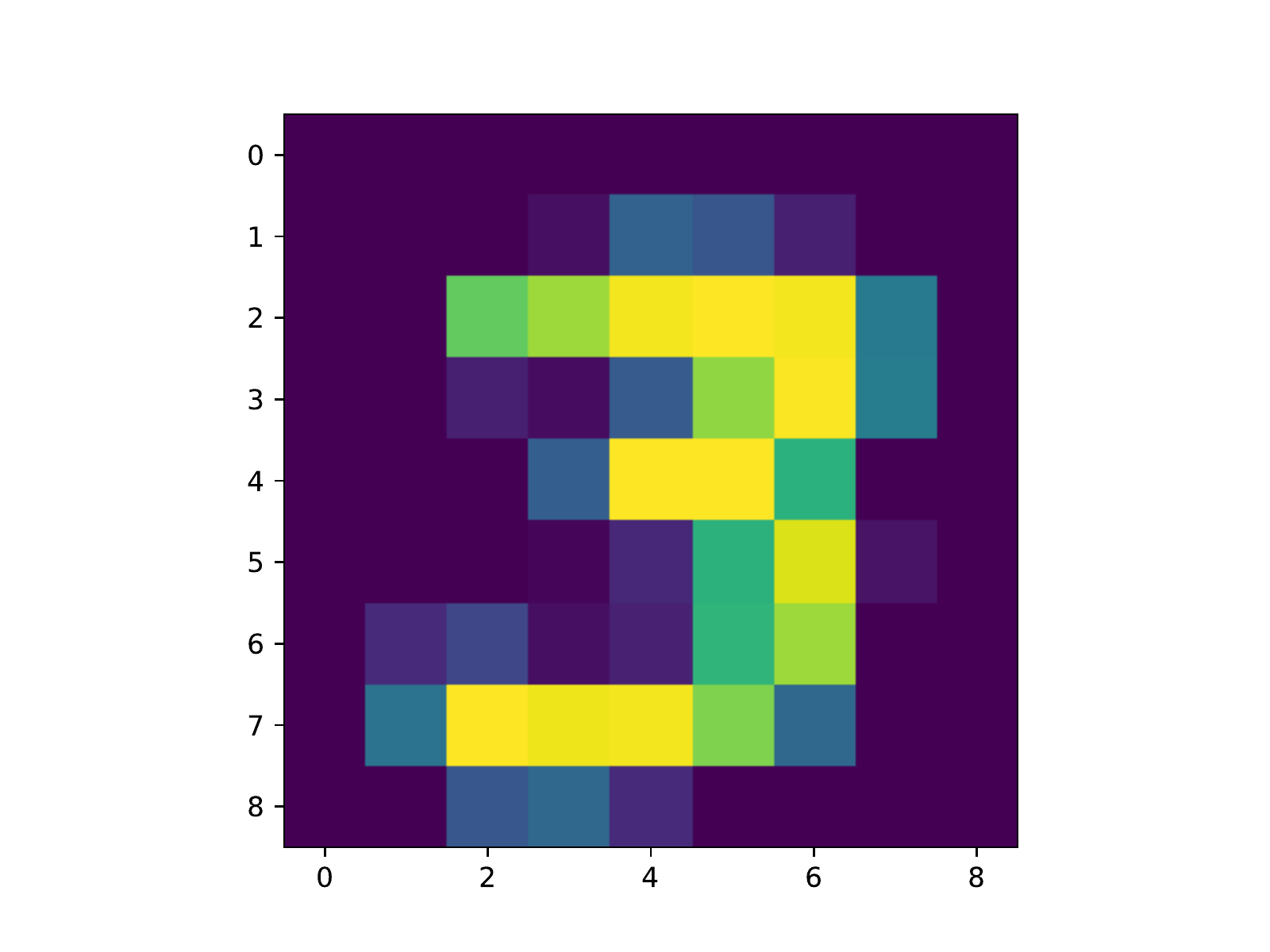} 
			\includegraphics[width=1.2in]{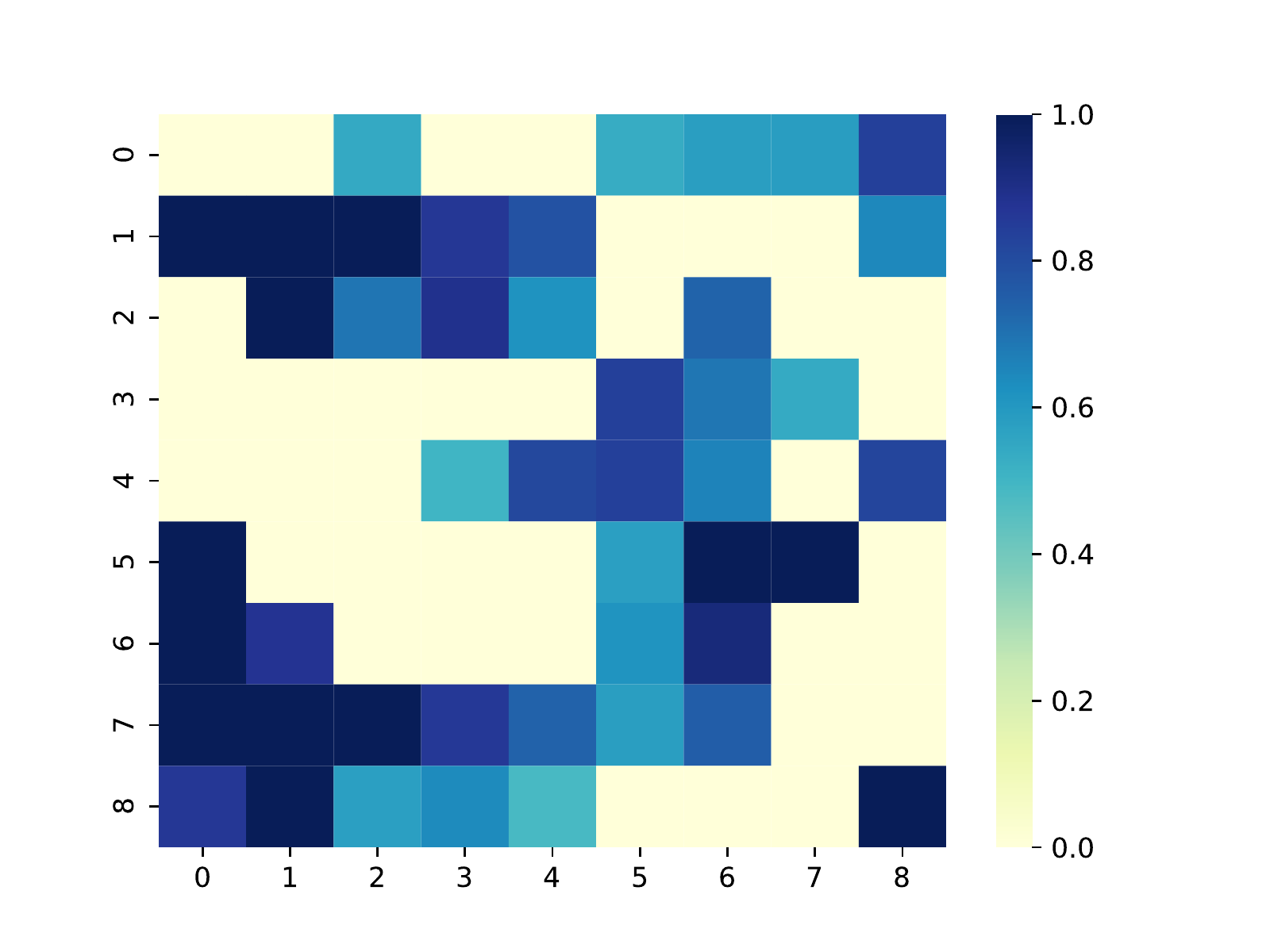}\\
			\includegraphics[width=1.2in]{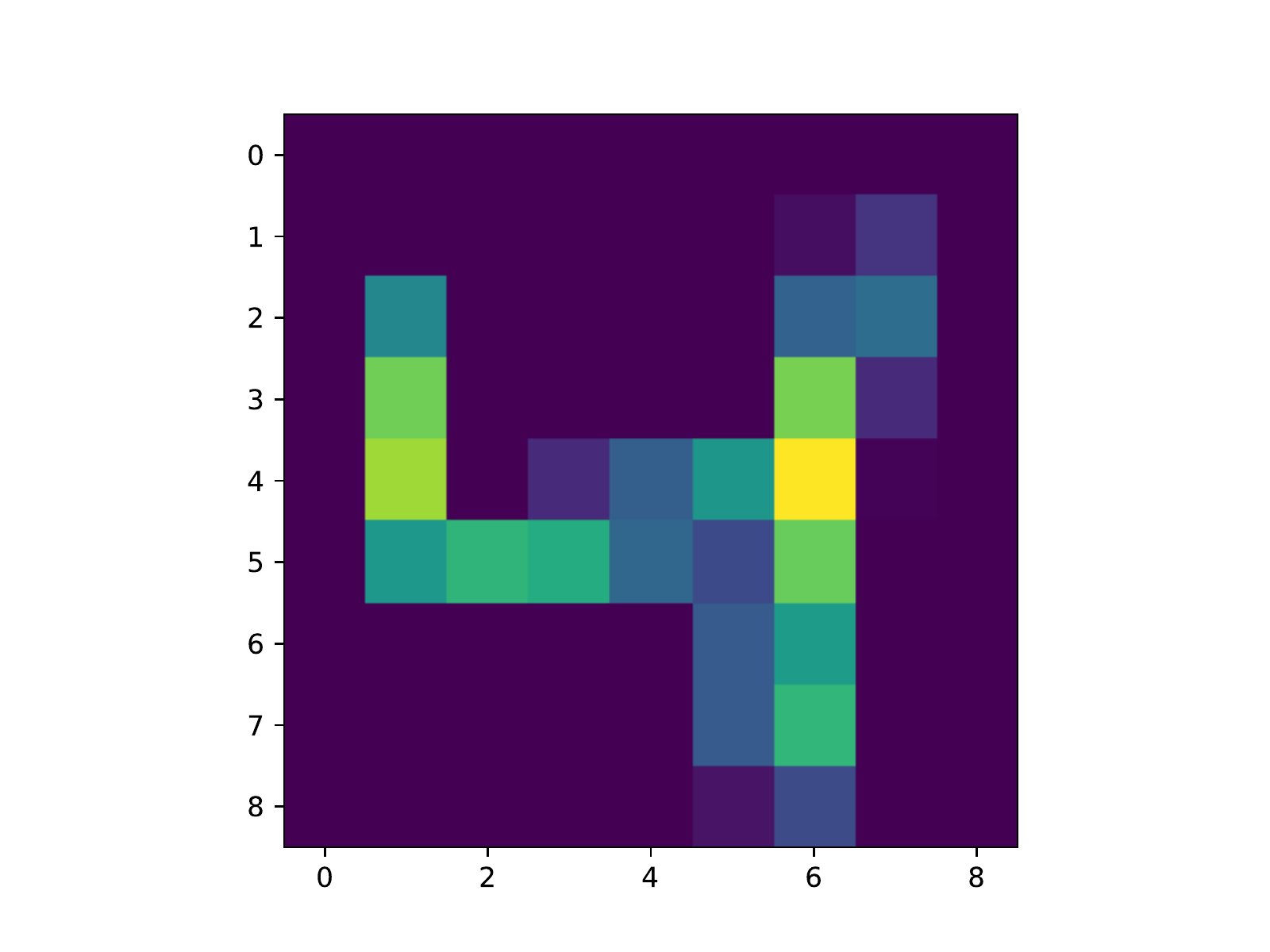}
			\includegraphics[width=1.2in]{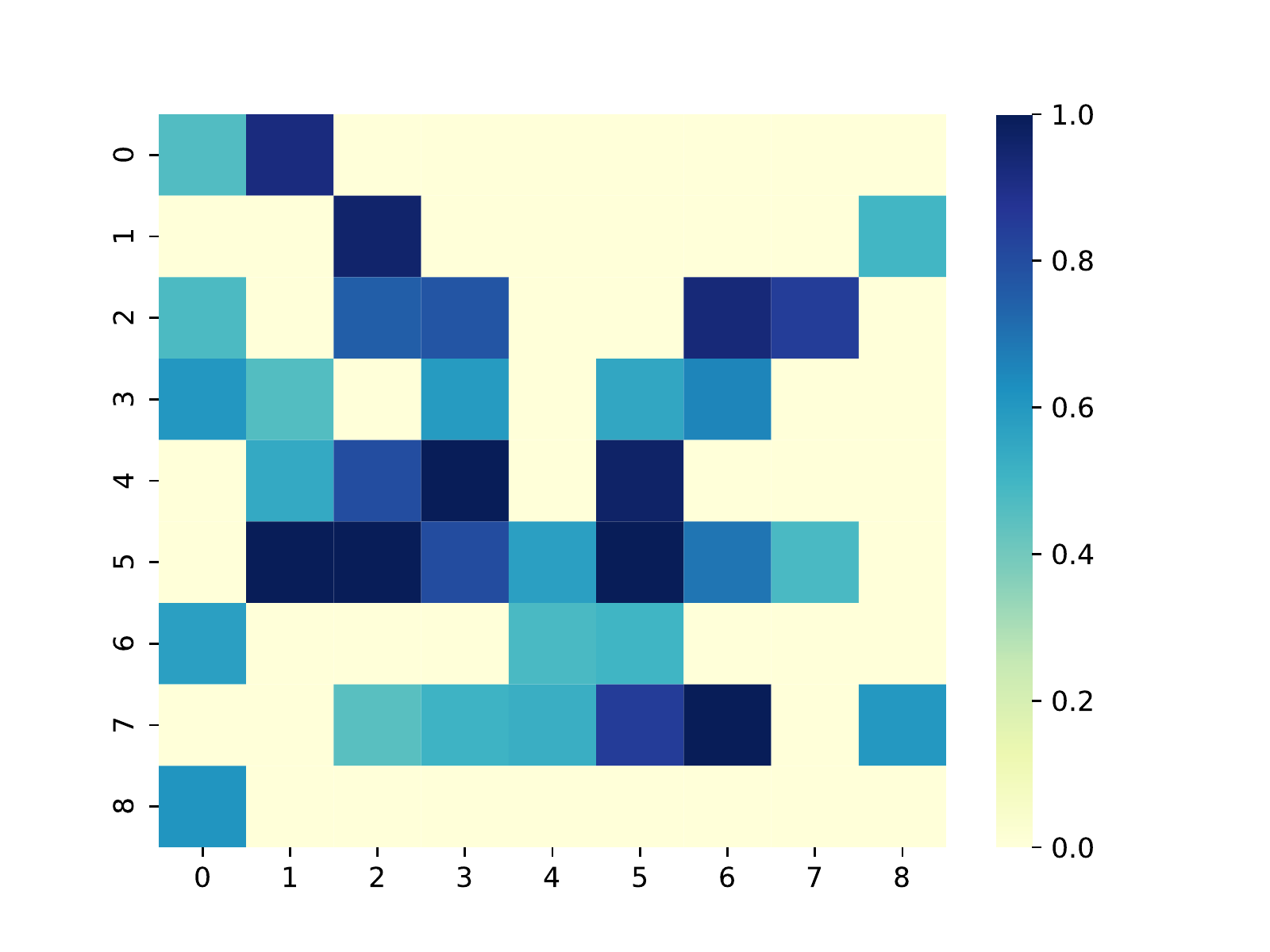}\\
			\includegraphics[width=1.2in]{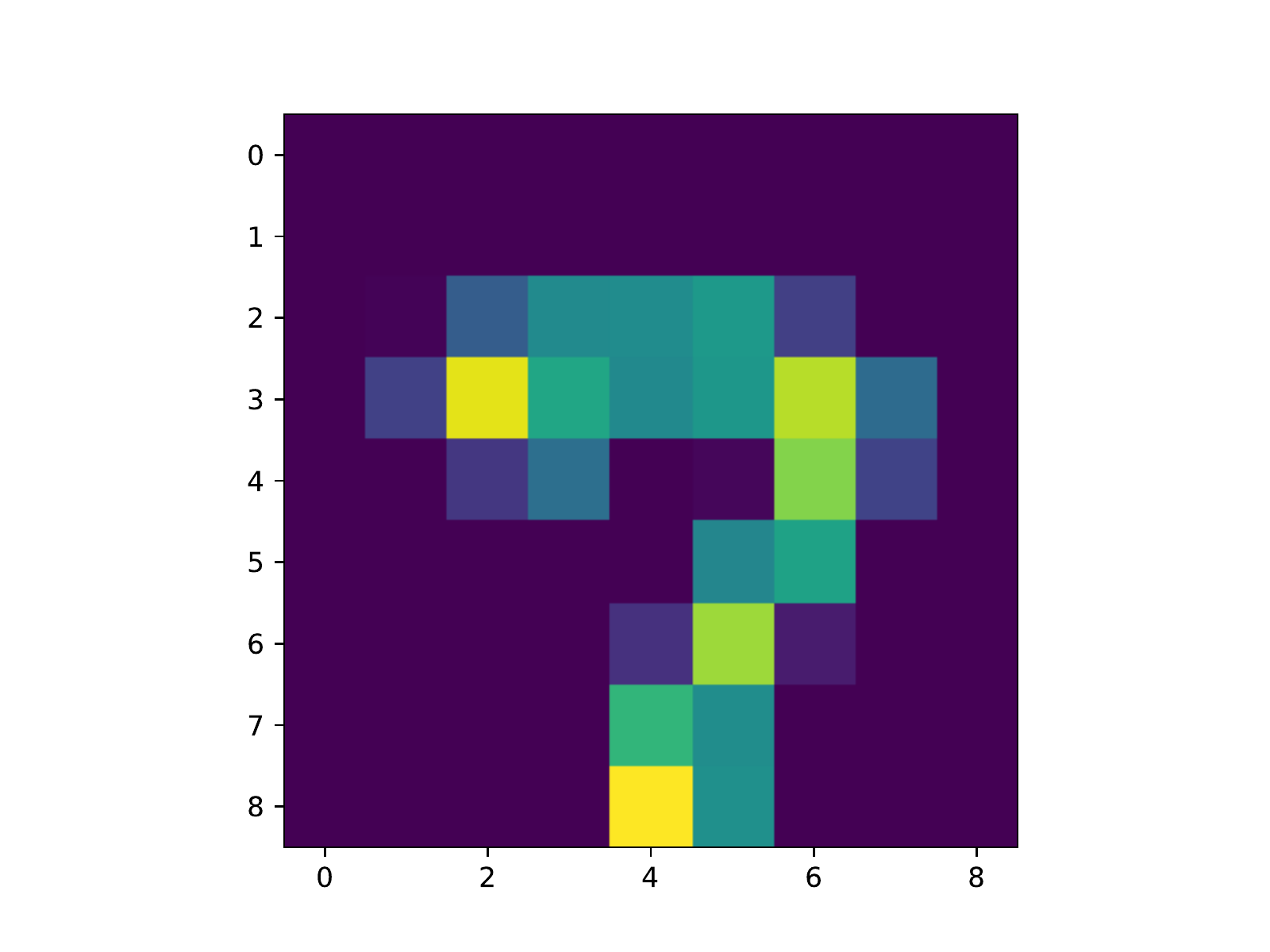}
			\includegraphics[width=1.2in]{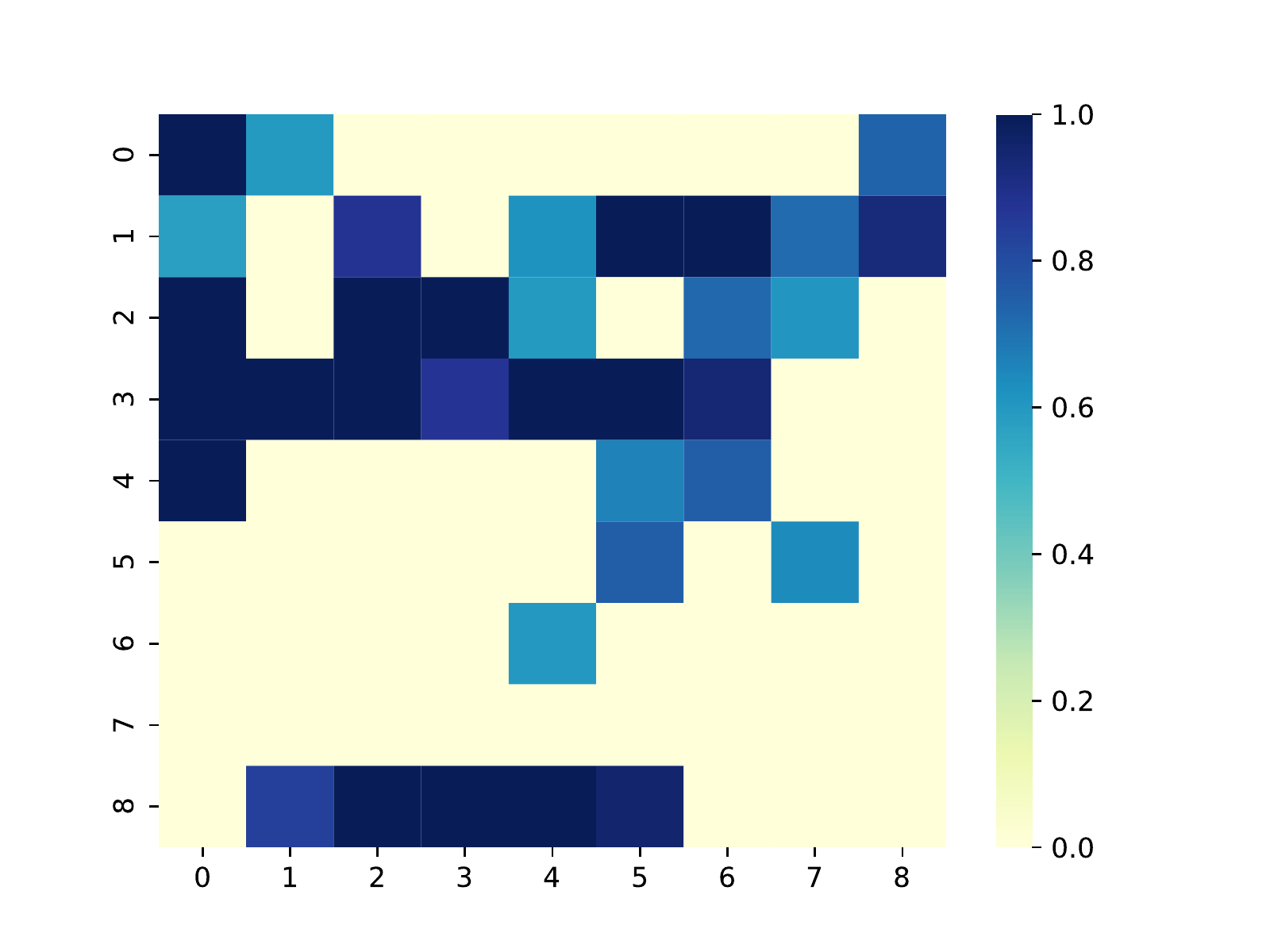}
		\label{fig:fea map}
		\end{minipage}
	}
	\caption {Illustrations on credit assignment and feature weight.}
\end{figure*}

\section{IV Experimental Demonstration}
\label{sec:exp}

In this section, we demonstrate the interpretability and effectiveness  of proposed method via experiments on MNIST dataset\footnote{\url{http://yann.lecun.com/exdb/mnist/}} and frequently used neural networks. All the experiments are run on the platform with 11 th Gen Intel(R) Core(TM) i7-11800H @ 2.3GHZ with RAM 16.0 GB.

\subsection{A. Component Credit Assignment.}
The neural network $\nn_1$ utilized in this subsection   is the  Lenet-5 neural network proposed in \cite{726791}, which is especially designed for handwritten numeral recognition of MNIST dataset in early years. The reason we do credit assignment on such a neural network is that it contains typical network layers: convolution, pooling and dense layers, whose architecture is shown in Fig. \ref{fig:lenet} (in gray color of the top part, also labelled with layer names and sizes). 

First, take an input image `5' as an illustration, heat maps displayed in Fig. \ref{fig:lenet} is the extracted feature maps from the adjacent network layers on the left and with the same sizes. Herein, we take each layer of Lenet-5 network as a block to be assigned credit. It can be observed that the feature maps of first convolution layer is  distinguishable for human, however, the following feature maps are difficult to recognize, let alone evaluating the layer credit according to them. 

Our credit assignment of each network layer and the convolution kernels of the two convolution layer (\textit{Conv1} and \textit{Conv2}) are shown at the bottom of Fig. \ref{fig:lenet}, labelled as \textit{Layer Credit}, \textit{Conv1 Kernel Credit} and \textit{Conv2 Kernel Credit} on the bottom of Fig. \ref{fig:lenet}. Notably, to obtain more exact Koopman operators for each block, the experiments are run 10 times and the mean matrix of  the result matrices are utilized for credit assignment. For each iteration, the matrix for dimension alignment is generated randomly with a check ensuring it satisfies the rank requirement, as declared in Theorem \ref{the align}. Compared with the feature maps, which are hard to understand as the number of network layers increases, our credit assignment method is more straightforward and rigorous. Taking the layer credit as an example, the result shows that convolution layers deserve more credit than dense layers and pooling layers receive least credit, which is coherent with human's perception and the design purpose of the network. Besides, the length of a bar in Fig. \ref{fig:lenet} represents the contribution of a component, whereas, it is only an abstract illustration  after processing the actual data and some closed values are assigned with the same credit rank.

\subsection{B. Feature Weight Assignment.}
Maybe the kernel credit assignment of convolution layers cannot be straightly confirmed by the confusing heat maps. We also have conducted an experiment to show the feature weight  from  $\bm{K}$ described before to further demonstrate the effectiveness of our credit assignment method.  Neural network $\nn_2$ used in this part is a fully-connected neural network instance, also designed and trained with MNIST dataset. It contains $5$ layers, with width $81$, $128$, $64$, $24$ and $10$ respectively, and for presentation clarity, we pool the input size from $28\times 28$ into $9\times 9$ with max pooling. Also, each layer is treated as a network block here.

In Fig. \ref{fig:fea map}, we show the input sample on the left column and the weights of each pixel utilized to do the classification decision in $\nn_2$ are displayed on the right column. The result is consistent with human recognition and it can be seen that the main feature pixels are captured by the weight, showing the validity of the linear approximation, rendering the subsequent credit assignment convincing to great extent. Similarly as above, the weight matrix also is obtained from 10 independent tests with randomly generated matrices for dimension alignment. 

\section{V Discussion}
\label{sec:discuss}

In this paper, we migrate the credit assignment problem to trained  artificial neural networks and take an alternative perspective of linear dynamics on tackling this problem. We regard an ANN as a dynamical system,  composed of some sequential sub-dynamics (components). To characterize the evolution mechanism of each block more comprehensively, we utilize the step-delay embedding to capture snapshots of each subsystem. More importantly,  during the embedding process, the minimal linear dimension alignment technique is put forward to resolve the dimension-difference problem between input and output layers of each block, which would inspire further insight into considering neural networks as dynamics. Furthermore, an algebraic credit metric is derived to evaluate the contribution of each component and  experiments on representative neural networks demonstrate the effectiveness of the proposed methods.

From the linear dynamic perspective,  our approach would provide convenience and guidance to (minimal) network patch, specification and verification, for its credit assignment on specific network layers or modules. Therefore, it is also necessary to do more validity test with practical application domains.  Meanwhile, the design and evaluation of more meaningful metrics for network components is also a promising future work.

\bibliography{aaai23}

\begin{thebibliography}{42}
\providecommand{\natexlab}[1]{#1}

\bibitem[{Akintunde et~al.(2019)Akintunde, Kevorchian, Lomuscio, and
  Pirovano}]{2019Verification}
Akintunde, M.~E.; Kevorchian, A.; Lomuscio, A.; and Pirovano, E. 2019.
\newblock Verification of RNN-Based Neural Agent-Environment Systems.
\newblock \emph{Proceedings of the AAAI Conference on Artificial Intelligence},
  33: 6006--6013.

\bibitem[{Bishop and Nasrabadi(2006)}]{bishop2006pattern}
Bishop, C.~M.; and Nasrabadi, N.~M. 2006.
\newblock \emph{Pattern recognition and machine learning}, volume~4.
\newblock Springer.

\bibitem[{Brunton et~al.(2017)Brunton, Brunton, Proctor, Kaiser, and
  Kutz}]{brunton2017chaos}
Brunton, S.~L.; Brunton, B.~W.; Proctor, J.~L.; Kaiser, E.; and Kutz, J.~N.
  2017.
\newblock Chaos as an intermittently forced linear system.
\newblock \emph{Nature communications}, 8(1): 1--9.

\bibitem[{Budi{\v{s}}i{\'c} et~al.(2012)Budi{\v{s}}i{\'c}, Marko, Mohr, and
  Mezi{\'c}}]{budivsic2012applied}
Budi{\v{s}}i{\'c}; Marko; Mohr, R.; and Mezi{\'c}, I. 2012.
\newblock Applied koopmanism.
\newblock \emph{Chaos: An Interdisciplinary Journal of Nonlinear Science},
  22(4): 047510.

\bibitem[{Dahnert et~al.(2021)Dahnert, Hou, Nie{\ss}ner, and
  Dai}]{dahnert2021panoptic}
Dahnert, M.; Hou, J.; Nie{\ss}ner, M.; and Dai, A. 2021.
\newblock Panoptic 3d scene reconstruction from a single rgb image.
\newblock \emph{Advances in Neural Information Processing Systems}, 34.

\bibitem[{Esser, Rombach, and Ommer(2021)}]{esser2021taming}
Esser, P.; Rombach, R.; and Ommer, B. 2021.
\newblock Taming transformers for high-resolution image synthesis.
\newblock In \emph{Proceedings of the IEEE/CVF Conference on Computer Vision
  and Pattern Recognition}, 12873--12883.

\bibitem[{Gers, Schmidhuber, and Cummins(2000)}]{gers2000learning}
Gers, F.~A.; Schmidhuber, J.; and Cummins, F. 2000.
\newblock Learning to forget: Continual prediction with LSTM.
\newblock \emph{Neural computation}, 12(10): 2451--2471.

\bibitem[{Goldberger et~al.(2020)Goldberger, Ben, Katz, Guy, Adi, and
  Keshet}]{goldberger2020minimal}
Goldberger; Ben; Katz; Guy; Adi, Y.; and Keshet, J. 2020.
\newblock Minimal Modifications of Deep Neural Networks using Verification.
\newblock In \emph{LPAR}, volume 2020, 23rd.

\bibitem[{Gover and Krikorian(2010)}]{gover2010determinants}
Gover, E.; and Krikorian, N. 2010.
\newblock Determinants and the volumes of parallelotopes and zonotopes.
\newblock \emph{Linear Algebra and its Applications}, 433(1): 28--40.

\bibitem[{Han et~al.(2015)Han, He, Bagchi, Fosler-Lussier, and
  Wang}]{han2015deep}
Han, K.; He, Y.; Bagchi, D.; Fosler-Lussier, E.; and Wang, D. 2015.
\newblock Deep neural network based spectral feature mapping for robust speech
  recognition.
\newblock In \emph{Sixteenth annual conference of the international speech
  communication association}.

\bibitem[{He et~al.(2016)He, Zhang, Ren, and Sun}]{he2016deep}
He, K.; Zhang, X.; Ren, S.; and Sun, J. 2016.
\newblock Deep residual learning for image recognition.
\newblock In \emph{Proceedings of the IEEE conference on computer vision and
  pattern recognition}, 770--778.

\bibitem[{Jacobi(1841)}]{jacobi1841determinantibus}
Jacobi, C. G.~J. 1841.
\newblock De Determinantibus functionalibus.
\newblock \emph{Journal f{\"u}r die reine und angewandte Mathematik (Crelles
  Journal)}, 1841(22): 319--359.

\bibitem[{James(1978)}]{james1978generalised}
James, M. 1978.
\newblock The generalised inverse.
\newblock \emph{The Mathematical Gazette}, 62(420): 109--114.

\bibitem[{Karch et~al.(2021)Karch, Teodorescu, Hofmann, Moulin-Frier, and
  Oudeyer}]{karch2021grounding}
Karch, T.; Teodorescu, L.; Hofmann, K.; Moulin-Frier, C.; and Oudeyer, P.-Y.
  2021.
\newblock Grounding Spatio-Temporal Language with Transformers.
\newblock \emph{arXiv preprint arXiv:2106.08858}.

\bibitem[{Kauschke and F{\"u}rnkranz(2018)}]{kauschke2018batchwise}
Kauschke, S.; and F{\"u}rnkranz, J. 2018.
\newblock Batchwise patching of classifiers.
\newblock In \emph{Proceedings of the AAAI Conference on Artificial
  Intelligence}, volume~32.

\bibitem[{Kawahara(2016)}]{kawahara2016dynamic}
Kawahara, Y. 2016.
\newblock Dynamic mode decomposition with reproducing kernels for Koopman
  spectral analysis.
\newblock \emph{Advances in neural information processing systems}, 29.

\bibitem[{Kingma and Ba(2015)}]{Kingma2015AdamAM}
Kingma, D.~P.; and Ba, J. 2015.
\newblock Adam: A Method for Stochastic Optimization.
\newblock \emph{CoRR}, abs/1412.6980.

\bibitem[{Koopman(1931)}]{koopman1931hamiltonian}
Koopman, B.~O. 1931.
\newblock Hamiltonian systems and transformation in Hilbert space.
\newblock \emph{Proceedings of the National Academy of Sciences}, 17(5):
  315--318.

\bibitem[{Kutz et~al.(2016)Kutz, Brunton, Brunton, and
  Proctor}]{kutz2016dynamic}
Kutz, J.~N.; Brunton, S.~L.; Brunton, B.~W.; and Proctor, J.~L. 2016.
\newblock \emph{Dynamic mode decomposition: data-driven modeling of complex
  systems}.
\newblock SIAM.

\bibitem[{Lecun et~al.(1998)Lecun, Bottou, Bengio, and Haffner}]{726791}
Lecun, Y.; Bottou, L.; Bengio, Y.; and Haffner, P. 1998.
\newblock Gradient-based learning applied to document recognition.
\newblock \emph{Proceedings of the IEEE}, 86(11): 2278--2324.

\bibitem[{Lengerich et~al.(2017)Lengerich, Konam, Xing, Rosenthal, and
  Veloso}]{2017Visual}
Lengerich, B.~J.; Konam, S.; Xing, E.~P.; Rosenthal, S.; and Veloso, M. 2017.
\newblock Visual Explanations for Convolutional Neural Networks via Input
  Resampling.
\newblock \emph{CoRR}, abs/1707.09641.

\bibitem[{Li et~al.(2017)Li, Dietrich, Bollt, and Kevrekidis}]{li2017extended}
Li, Q.; Dietrich, F.; Bollt, E.~M.; and Kevrekidis, I.~G. 2017.
\newblock Extended dynamic mode decomposition with dictionary learning: A
  data-driven adaptive spectral decomposition of the Koopman operator.
\newblock \emph{Chaos: An Interdisciplinary Journal of Nonlinear Science},
  27(10): 103111.

\bibitem[{Liu et~al.(2021)Liu, Arnon, Lazarus, Strong, Barrett, Kochenderfer
  et~al.}]{liu2021algorithms}
Liu, C.; Arnon, T.; Lazarus, C.; Strong, C.; Barrett, C.; Kochenderfer, M.~J.;
  et~al. 2021.
\newblock Algorithms for verifying deep neural networks.
\newblock \emph{Foundations and Trends{\textregistered} in Optimization},
  4(3-4): 244--404.

\bibitem[{Liu et~al.(2020)Liu, Song, Zhang, and Wang}]{liu2020verifying}
Liu, W.-W.; Song, F.; Zhang, T.-H.-R.; and Wang, J. 2020.
\newblock Verifying ReLU neural networks from a model checking perspective.
\newblock \emph{Journal of Computer Science and Technology}, 35(6): 1365--1381.

\bibitem[{Mezi{\'c}(2005)}]{mezic2005spectral}
Mezi{\'c}, I. 2005.
\newblock Spectral properties of dynamical systems, model reduction and
  decompositions.
\newblock \emph{Nonlinear Dynamics}, 41(1): 309--325.

\bibitem[{Minsky(1961)}]{minsky1961steps}
Minsky, M. 1961.
\newblock Steps toward artificial intelligence.
\newblock \emph{Proceedings of the IRE}, 49(1): 8--30.

\bibitem[{Nemirovski et~al.(2009)Nemirovski, Juditsky, Lan, and
  Shapiro}]{sgdarticle}
Nemirovski, A.; Juditsky, A.; Lan, G.; and Shapiro, A. 2009.
\newblock Robust Stochastic Approximation Approach to Stochastic Programming.
\newblock \emph{Society for Industrial and Applied Mathematics}, 19:
  1574--1609.

\bibitem[{Penrose(1955)}]{penrose1955generalized}
Penrose, R. 1955.
\newblock A generalized inverse for matrices.
\newblock In \emph{Mathematical proceedings of the Cambridge philosophical
  society}, volume~51, 406--413. Cambridge University Press.

\bibitem[{Planitz(1979)}]{planitz19793}
Planitz, M. 1979.
\newblock 3. Inconsistent systems of linear equations.
\newblock \emph{The Mathematical Gazette}, 63(425): 181--185.

\bibitem[{Sauer, Yorke, and Casdagli(1991)}]{sauer1991embedology}
Sauer, T.; Yorke, J.~A.; and Casdagli, M. 1991.
\newblock Embedology.
\newblock \emph{Journal of statistical Physics}, 65(3): 579--616.

\bibitem[{Takeishi, Kawahara, and Yairi(2017)}]{takeishi2017learning}
Takeishi, N.; Kawahara, Y.; and Yairi, T. 2017.
\newblock Learning Koopman invariant subspaces for dynamic mode decomposition.
\newblock \emph{Advances in Neural Information Processing Systems}, 30.

\bibitem[{Takens(1981)}]{takens1981detecting}
Takens, F. 1981.
\newblock Detecting strange attractors in turbulence.
\newblock In \emph{Dynamical systems and turbulence, Warwick 1980}, 366--381.
  Springer.

\bibitem[{Tang et~al.(2018)Tang, Haijing, Wang, Yiru, and
  Yang}]{tang2018evaluation}
Tang; Haijing; Wang; Yiru; and Yang, X. 2018.
\newblock Evaluation of Visualization Methods' Effect on Convolutional Neural
  Networks Research.
\newblock In \emph{Proceedings of the 2018 International Conference on
  Algorithms, Computing and Artificial Intelligence}, 1--5.

\bibitem[{Tenenbaum, Silva, and Langford(2000)}]{tenenbaum2000global}
Tenenbaum, J.~B.; Silva, V.~d.; and Langford, J.~C. 2000.
\newblock A global geometric framework for nonlinear dimensionality reduction.
\newblock \emph{science}, 290(5500): 2319--2323.

\bibitem[{Tian, Yang, and Wang(2021)}]{tian2021image}
Tian, Y.; Yang, W.; and Wang, J. 2021.
\newblock Image fusion using a multi-level image decomposition and fusion
  method.
\newblock \emph{Applied Optics}, 60(24): 7466--7479.

\bibitem[{Tieleman, Hinton et~al.(2012)}]{tieleman2012lecture}
Tieleman, T.; Hinton, G.; et~al. 2012.
\newblock Lecture 6.5-rmsprop: Divide the gradient by a running average of its
  recent magnitude.
\newblock \emph{COURSERA: Neural networks for machine learning}, 4(2): 26--31.

\bibitem[{Vengertsev and Sherman(2020)}]{vengertsev2020recurrent}
Vengertsev, D.; and Sherman, E. 2020.
\newblock Recurrent Neural Network Properties and their Verification with Monte
  Carlo Techniques.
\newblock In \emph{Proceedings of the AAAI Conference on Artificial
  Intelligence}, volume 2560, 178--185.

\bibitem[{Wang et~al.(2021)Wang, Wang, Rudzicz, and Brudno}]{wang2021grad2task}
Wang, J.; Wang, K.-C.; Rudzicz, F.; and Brudno, M. 2021.
\newblock Grad2Task: Improved Few-shot Text Classification Using Gradients for
  Task Representation.
\newblock \emph{Advances in Neural Information Processing Systems}, 34.

\bibitem[{Werbos(1974)}]{werbos1974beyond}
Werbos, P. 1974.
\newblock Beyond regression:" new tools for prediction and analysis in the
  behavioral sciences.
\newblock \emph{Ph. D. dissertation, Harvard University}.

\bibitem[{Whitney(1936)}]{whitney1936differentiable}
Whitney, H. 1936.
\newblock Differentiable manifolds.
\newblock \emph{Annals of Mathematics}, 645--680.

\bibitem[{Wold, Esbensen, and Geladi(1987)}]{wold1987principal}
Wold, S.; Esbensen, K.; and Geladi, P. 1987.
\newblock Principal component analysis.
\newblock \emph{Chemometrics and intelligent laboratory systems}, 2(1-3):
  37--52.

\bibitem[{Yuan, Neubig, and Liu(2021)}]{yuan2021bartscore}
Yuan, W.; Neubig, G.; and Liu, P. 2021.
\newblock BARTScore: Evaluating Generated Text as Text Generation.
\newblock \emph{arXiv preprint arXiv:2106.11520}.

\end{thebibliography}

\end{document}